\documentclass[11pt]{article}

\date{}

\usepackage[numbers, compress]{natbib}
\usepackage{graphicx}
\usepackage{subfigure}
\usepackage[explicit]{titlesec}
\titlespacing{\paragraph}{0pt}{0pt}{.25em}[]
\usepackage{amsmath,amsfonts,bm}

\def\eqref#1{equation~\ref{#1}}
\def\Eqref#1{Equation~\ref{#1}}
\def\1{\bm{1}}

\DeclareMathAlphabet{\mathsfit}{\encodingdefault}{\sfdefault}{m}{sl}
\SetMathAlphabet{\mathsfit}{bold}{\encodingdefault}{\sfdefault}{bx}{n}

\usepackage{enumitem}
\usepackage{hyperref}
\usepackage{url}
\usepackage{xcolor}
\usepackage{algorithm}
\usepackage{algorithmic}
\usepackage{mathtools}
\usepackage{caption}
\definecolor{lightblue}{rgb}{0.63, 0.74, 0.78}
\definecolor{seagreen}{rgb}{0.18, 0.42, 0.41}
\definecolor{orange}{rgb}{0.85, 0.55, 0.13}
\definecolor{silver}{rgb}{0.69, 0.67, 0.66}
\definecolor{rust}{rgb}{0.72, 0.26, 0.06}

\colorlet{lightsilver}{silver!30!white}
\colorlet{darkorange}{orange!85!black}
\colorlet{darksilver}{silver!85!black}
\colorlet{darklightblue}{lightblue!85!black}
\colorlet{darkrust}{rust!85!black}

\hypersetup{colorlinks=true,linkcolor=darkrust,citecolor=darklightblue,urlcolor=darksilver}
\usepackage{amsthm}
\newtheorem{theorem}{Theorem}

\newtheorem{proposition}[theorem]{Proposition}
\newtheorem{definition}[theorem]{Definition}
\newtheorem{example}{Example}

\usepackage{tikz}
\usetikzlibrary{shapes.geometric, arrows,positioning}

\tikzset{
  startstop/.style={
    rectangle, 
    rounded corners,
    minimum width=3cm, 
    minimum height=1cm,
    align=center, 
    draw=black
    },
  process/.style={
    rectangle, 
    minimum width=3cm, 
    minimum height=1cm, 
    align=center, 
    draw=black
    },
  decision/.style={
    rectangle, 
    minimum width=3cm, 
    minimum height=1cm, align=center, 
    draw=black
    },
  arrow/.style={draw,thick,->,>=stealth},
  dec/.style={
    ellipse, 
    align=center, 
    draw=black
    },
}

\title{Polymatrix Competitive Gradient Descent}

\author{Jeffrey Ma\thanks{Jeffrey Ma is with California Institute of Technology. {\tt \small jma@caltech.edu}}, Alistair Letcher\thanks{Alistair Letcher, {\tt \small ahp.letcher@gmail.com}}, Florian Sch{\"a}fer\thanks{Florian Sch{\"a}fer is with the Georgia Institute of Technology, {\tt \small florian.schaefer@cc.gatech.edu}}, 
Yuanyuan Shi\thanks{Yuanyuan Shi is with University of California San Diego, {\tt\small yyshi@eng.ucsd.edu}}%
, and Anima Anandkumar\thanks{Anima Anandkumar is with Nvidia and California Institute of Technology, {\tt \small anima@caltech.edu}}% <-this % stops a space
}

\begin{document}
\maketitle

\begin{abstract}
    Many economic games and machine learning approaches can be cast as \emph{competitive optimization} problems where multiple agents are minimizing their respective objective function, which depends on all agents' actions. While gradient descent is a reliable basic workhorse for single-agent optimization, it often leads to oscillation in competitive optimization.
    In this work we propose polymatrix competitive gradient descent (PCGD) as a method for solving general sum competitive optimization involving arbitrary numbers of agents.
    The updates of our method are obtained as the Nash equilibria of a local polymatrix approximation with a quadratic regularization, and can be computed efficiently by solving a linear system of equations. 
    We prove local convergence of PCGD to stable fixed points for $n$-player general-sum games, and show that it does not require adapting the step size to the strength of the player-interactions.
    %\todo{Need to modify the following, right now it's a bit overselling the results on electricity} 
    We use PCGD to optimize policies in multi-agent reinforcement learning and demonstrate its advantages in Snake, Markov soccer and an electricity market game. Agents trained by PCGD outperform agents trained with simultaneous gradient descent, symplectic gradient adjustment, and extragradient in Snake and Markov soccer games and on the electricity market game, PCGD trains faster than both simultaneous gradient descent and the extragradient method.
\end{abstract}
\section{Introduction}
% Provide a few examples, mention related work (CGD, CMD, + ICR)
\paragraph{Multi-agent optimization} is used to model cooperative
and competitive behaviors of a group of interacting agents, each minimizing their respective objective function. These problems arise naturally in various applications, including robotics~\cite{serrino2019finding}, distributed control~\cite{marden2015game} and socio-economic systems~\cite{yu2007modeling,ruhi2017opportunities}.
More recently, multi-agent optimization has served as a powerful paradigm for the design of novel algorithms including generative modeling~\cite{goodfellow2014generative}, adversarial robustness \cite{madry2017towards} and uncertainty quantification~\cite{wang2020distributionally}, as well as multi-agent reinforcement learning (MARL).

\paragraph{Simultaneous gradient descent and the cycling problem.} 
The most common approach for solving multi-agent optimization problems is \emph{simultaneous gradient descent} (SimGD). In SimGD, all players independently change their strategy in the direction of steepest descent of their own cost function. However, this procedure often fails to converge, including in bilinear two-player games, where SimGD exhibits oscillatory and diverging behavior. 
This problem is addressed by numerous modifications of SimGD that are introduced through the lens of agent-behavior \citep{shalev2007convex, brown1951iterative,daskalakis2017training,foerster2018learning,mescheder2017numerics}, variational inequalities \citep{korpelevich1977extragradient}, or the Helmholtz-decomposition \citep{balduzzi2018mechanics,ramponi2021newton}. 
However, most of these methods require step sizes that are chosen inversely proportional to the strength of agent-interaction, limiting their convergence speed and stability.

% Many existing works attempt to stabilize SimGD through heuristics that help players predict each others' actions thus to simulate the interactive behavior, such as methods based on follow the regularized leader~\citep{shalev2007convex}, fictitious play~\citep{brown1951iterative}, opponent learning
% awareness~\citep{foerster2018learning} and optimism~\citep{mertikopoulos2019optimistic}, which is closely related \todo{[TODO: What is this close relation?]} to the extragradient method~\citep{korpelevich1977extragradient}. 
% Other methods reward agents for convergence \citep{mescheder2017numerics} or use the Helmholtz-decomposition \citet{balduzzi2018mechanics,ramponi2021newton}. \todo{[TODO: Move the Helmholtz decomposition to a later area.]}
% However, many of these methods require the stepsize to be chosen inversely proportional to the interaction between players, requiring repeated trial and error for hyperparameter tuning in practice.
% 

\paragraph{Competitive gradient descent (CGD).} The updates of SimGD are Nash equilibria of a \emph{local game} given by a regularized linear approximation of the agents' losses.
The authors of \cite{schafer2019competitive} argue that the shortcomings of SimGD arise from the inability of the linear approximation to account for agent interactions. For two-player optimization, they propose to use a \emph{bilinear} approximation that allows the agents to take each other's objectives into account when choosing an update.
They show that the Nash equilibrium of the resulting regularized bilinear game has a closed form solution which provides the update rule of a new algorithm, named \emph{competitive gradient descent}(CGD). Both theoretically and empirically, CGD shows improved stability properties and convergence speed, decoupling the admissible step size from the strength of player interactions. However, the algorithm presented therein is only applicable to the two-player setting. 

\paragraph{CGD for more than two players.} Since the two-agent CGD update is given by the Nash equilibrium of a regularized bilinear approximation of both agents' cost functions, a natural extension to $n$-agent optimization would require an $n$-th order multilinear approximation to capture all degrees of player interaction. 
% \todo{better notation for multilinear}
However, this approach requires the solution of a system of $(n-1)$-th order multilinear equations at each step. 
For $n = 2$, the resulting linear system of equations can be solved efficiently using the well developed tools from numerical linear algebra.
In sharp contrast, the solution of $(>\! 1)$-th order multilinear systems of equations is NP-hard in general \citet{hillar2013most}. 
Furthermore, existing heuristics such as alternating least squares (ALS) to solve multilinear systems do not offer nearly the same reliability, ease-of-use, and efficiency of optimal Krylov subspace methods available for systems of linear equations.
In order to avoid this difficulty, we propose to use a local approximation given by a multilinear \emph{polymatrix game}, where only interactions between pairs of agents are accounted for explicitly.
This local approximation can then be solved using linear algebraic methods.

\paragraph{Our contributions.} In this work, we introduce polymatrix competitive gradient descent (PCGD) as a natural extension of gradient descent to multi-agent competitive optimization. The updates of PCGD are given by the Nash equilibria of a regularized polymatrix approximation of the local game. 
This approximation allows us to preserve the advantages of two-agent CGD in the multiagent setting, while still computing updates using the powerful tools of numerical linear algebra.
\begin{itemize}[leftmargin=*]
    \item On the theoretical side, we prove local convergence of PCGD in the vicinity of local Nash equilibria, for multi-agent general-sum games. These results generalize the convergence results of CGD \cite{schafer2019competitive} from two-player zero-sum to $n$-player general-sum games. In particular, PCGD guarantees convergence without needing to adapt the step sizes to the strength of agent interactions. Existing approaches based on simultaneous gradient descent need to reduce the step size to match the increases in competitive interactions to avoid divergence, requiring more hyperparameter tuning.
    \item On the empirical side, we use PCGD for policy optimization in multi-agent reinforcement learning, and demonstrate its advantage in four-player snake and Markov soccer games. It is shown that, agents trained with PCGD, significantly outperform their SimGD, SGA, and Extragradient trained counterparts. 
    In particular, PCGD trained agents win at more than twice the rate compared to SimGD, SGA, and Extragradient trained agents. This holds even in settings where the majority of agents are trained using non-PCGD methods, meaning the PCGD agent encounters a major distributional shift. 
    %\todo{[TODO: Summarize other baseline results a bit more here.]}
    
    \item We also use PCGD for simulating strategic behavior in an economic market game that captures the essence of the real-world electricity market operating rules. 
    %\todo{Need to change language, this is overselling our results}
    We observe that PCGD improves the speed and stability of training, while reaching comparable reward.
    Thus, we can use PCGD as an effective tool to explore the behavior of self-interested agents in markets.
\end{itemize}

\section{Polymatrix Competitive Gradient Descent}
\label{sec:pcgd}
\paragraph{Setting and notation.} We consider general multi-agent competitive optimization of the form
\begin{equation}\label{eq:multiagent_opt}
\forall i \in \{1, ..., n\}\,, \min_{\theta^i \in \mathbb{R}^{d_i}} L^i(\theta^1, ..., \theta^n)
\end{equation}
for $n$ functions $L^1, L^2, ..., L^n: \mathbb{R}^{d} \rightarrow \mathbb{R}$, where $d =\sum d_i$.
We denote the \emph{combined parameter vector} of all players as $\theta \coloneqq (\theta^1, ..., \theta^n)$, the \emph{simultaneous gradient} as $\xi(\theta) \coloneqq \left(\nabla_{1} L^1(\theta), ..., \nabla_{n} L^n(\theta)\right)^\top$, and the \emph{game Hessian} as 
\begin{equation}
    H(\theta) = \nabla{\xi(\theta)} = \begin{bmatrix}
        \nabla_{11}{L^1}(\theta) & \cdots & \nabla_{1n}{L^1}(\theta) \\
        \vdots & \ddots & \vdots \\
        \nabla_{n1}{L^n}(\theta) & \cdots & \nabla_{nn}{L^n}(\theta)
    \end{bmatrix} \in \mathbb{R}^{d \times d}.
    \label{metamatrix}
\end{equation}
We denote as $H_d(\theta)$ the block-diagonal part of $H(\theta)$, and as $H_o(\theta)$ its block-off-diagonal part.

\paragraph{The multilinear polymatrix approximation.} Where SimGD and CGD compute updates as Nash equilibria of linear or multilinear approximations, we propose to instead use a multilinear \emph{polymatrix} approximation of the form 
\begin{equation}\label{eq:n_linear}
    L^{i}(\theta_k + \vec{\theta}) \approx L^i(\theta_k) + \vec{\theta}^T \nabla L^i(\theta_k) + \sum_{j \neq i} \vec{\theta}^{i,\top} \nabla_{ij} L^i(\theta_k) \vec{\theta}^{j}.
\end{equation}
%\todo{Should we still add a motivational sentence here why we are using it? I just gave it a shot but found it to kind of break the flow so I tend towards keeping it as-is right now}
By adding a quadratic regularization that expresses the limited confidence of the players in the accuracy of the local approximation for large $\vec{\theta}$, we obtain the local polymatrix game
\begin{equation}\label{eq:n_linear_reg}
\forall i \in \{1, ..., n\}\,, \min_{\vec{\theta}^i \in \mathbb{R}^{d_i}} L^i + \vec{\theta}^{i, T} \nabla L^i + \sum_{j \neq i} \vec{\theta}^{i,\top} \nabla_{ij} L^i \vec{\theta}^j + \frac{1}{2\eta} \vec{\theta}^{i, \top} \vec{\theta}^{i}.
\end{equation}
Here and from now on, we stop explicitly denoting the dependence on the last iterate $\theta_k$ in order to simplify the notation.
We will now derive the unique Nash equilibrium of this game.
\begin{proposition}
The game in \eqref{eq:n_linear_reg} has a unique Nash equilibrium given by
\begin{equation}
    \label{eq:nash}
    \vec{\theta} = -\eta (I+\eta H_o)^{-1} \xi\,,
\end{equation}
provided $\eta$ is sufficiently small for the inverse matrix to exist.
\end{proposition}
\begin{proof}
A necessary condition for $\vec{\theta}$ to be the Nash equilibrium is that for each $i$, 
$$\nabla_i L^i + \sum_{j \neq i} \nabla_{ij} L^i \vec{\theta}^j + \frac{1}{\eta} \vec{\theta}^i = \nabla_i L^i + \sum_{j} (H_o)_{ij} \vec{\theta}^j + \frac{1}{\eta} \vec{\theta}^i = \left(\xi+H_o \vec{\theta} +\frac{1}{\eta} \vec{\theta}\right)_i = 0.$$
Therefore,
$$\vec{\theta} = -\eta (I+\eta H_o)^{-1} \xi $$
is the unique possible solution for $\eta > 0$ small. This is a Nash equilibrium since
$$\nabla_{ii} \left[\vec{\theta}^\top \nabla L^i + \sum_{j \neq i} {\vec{\theta}^{i, \top}} \nabla_{ij} L^i \vec{\theta}^j + \frac{1}{2\eta} \vec{\theta}^\top \vec{\theta}\right] = \nabla_{i} \left[\frac{1}{\eta} \vec{\theta}^i\right] = \frac{1}{\eta} I \succ 0 \,,$$
everywhere, a sufficient condition for fixed points to be Nash equilibria.
\end{proof}

\paragraph{Polymatrix competitive gradient descent (PCGD)}
is obtained by using the Nash equilibrium $\vec{\theta}$ of \eqref{eq:nash} as an update rule according to $\theta_{k+1} = \theta_k + \vec{\theta}$.
% Note on recalculating the gradient and hessian after each update?
\begin{algorithm}[ht]
   \caption{Polymatrix Competitive Gradient Descent (PCGD)}
   \label{alg:pcgd}
   \textbf{Input: } objectives $\{L^1(\theta), ..., L^n(\theta)\}$, parameters $\theta_{0}$
    \begin{algorithmic}
    \FOR{$0 \leq k \leq N - 1$}
       \STATE $\theta_{k+1} =  \theta_{k} -\eta (I + \eta H_o)^{-1}\xi$
    \ENDFOR
    \end{algorithmic}
\textbf{Output:} $\theta_{N}$
\end{algorithm}

By incorporating the interaction between the different players at each step, PCGD avoids the cycling problem encountered by simultaneous gradient descent.
In particular, we will show in Section~\ref{sec:theory} that the local convergence of PCGD is robust to arbitrarily strong competitive interactions between the players, as described by the antisymmetric part of $H$.

\paragraph{Why not use multi-linear approximation?} The authors of \cite{schafer2019competitive} propose to extend CGD to more than two players by using a full multilinear approximation of the objective function.
% \todo{[TODO: Include a diagram that shows interaction as a cube in three player case, where edges are pairwise interactions]}
For example, in a three-player game with objectives $L^1, L^2, L^3: \mathbb{R}^{d_1} \times \mathbb{R}^{d_2} \times \mathbb{R}^{d_3} \rightarrow \mathbb{R}$, the update would be the solution $\theta = ({\theta^1}, {\theta^2}, {\theta^3})$ to the following game:
\begin{equation}\label{eq:trilinear}
    \begin{gathered}
        \min_{\theta^1 \in \mathbb{R}^{d_1}}{
            \left\{
            \substack{
                 {\theta^1}^\top \nabla_{1} L^1 + {\theta^2}^\top \nabla_2 L^1 + {\theta^3}^\top \nabla_3 L^1 + 
                 {\theta^1}^\top [\nabla_{12} L^1] {\theta^2} + {\theta^1}^\top [\nabla_{13} L^1] {\theta^3} + 
                 {\theta^2}^\top [\nabla_{23} L^1] {\theta^3} \\
                 + [\nabla_{123} L^1] \times_1 [{\theta^1}] \times_2 [{\theta^2}] \times_3 [{\theta^3}] + \frac{1}{2\eta} {\theta^1}^\top {\theta^1}
            }
            \right\}
        } \\
        \min_{\theta^2 \in \mathbb{R}^{d_2}}{
            \left\{
            \substack{
                 {\theta^1}^\top \nabla_1 L^2 + {\theta^2}^\top \nabla_2 L^2 + {\theta^3}^\top \nabla_3 L^2 
                 + {\theta^1}^\top [\nabla_{12} L^2] {\theta^2} + {\theta^1}^\top [\nabla_{13} L^2] {\theta^3} + {\theta^2}^\top [\nabla_{23} L^2] {\theta^3} \\ 
                 + [\nabla_{123} L^2] \times_1 [{\theta^1}] \times_2 [{\theta^2}] \times_3 [{\theta^3}] + \frac{1}{2\eta}{\theta^2}^\top {\theta^2}
            }
            \right\}
        } \\
        \min_{\theta^3 \in \mathbb{R}^{d_3}}{
            \left\{
            \substack{
                 {\theta^1}^\top \nabla_1 L^3 + {\theta^2}^\top \nabla_2 L^3 + {\theta^3}^\top \nabla_3 L^3  
                 + {\theta^1}^\top [\nabla_{12} L^3] {\theta^2} + {\theta^1}^\top [\nabla_{13} L^3] {\theta^3} + {\theta^2}^\top [\nabla_{23} L^3] {\theta^3} \\ 
                 + [\nabla_{123} L^3] \times_1 [{\theta^1}] \times_2 [{\theta^2}] \times_3 [{\theta^3}] + \frac{1}{2\eta}{\theta^3}^\top {\theta^3}
            }
            \right\}.
        } 
    \end{gathered}
\end{equation}

As shown in the trilinear approximation in \eqref{eq:trilinear}, we start to consider interactions between groups of three or more agents captured by higher dimensional interaction tensors such as $\nabla_{123} L^1$ for each update. 
The optimality conditions obtained by differentiating the $i$-th loss of the local game with respect to $\theta^i$ are given by a general system of multilinear equations and even deciding if a solution exists is NP-hard in general \citep{hillar2013most}. Existing approaches such as alternating least squares do not offer nearly the same reliability, ease-of-use, and efficiency as the well-developed tools for systems of linear equations.
In contrast, the multilinear \emph{polymatrix} approximation of \eqref{eq:n_linear} specializes to 
\begin{equation}
    \begin{gathered}
        \min_{\theta^1 \in \mathbb{R}^{d_1}}{
            {\theta^1}^\top \nabla_1 L^1 + {\theta^1}^\top \nabla_{12} L^1 {\theta^2} + {\theta^1}^\top \nabla_{13} L^1 {\theta^3} + {\theta^2}^\top \nabla_2 L^1 + {\theta^3}^\top \nabla_3 L^1 + \frac{1}{2\eta}{\theta^1}^\top {\theta^1}
        } \\
        \min_{\theta^2 \in \mathbb{R}^{d_2}}{
            {\theta^2}^\top \nabla_2 L^2 + {\theta^2}^\top \nabla_{21} L^2 {\theta^1} + {\theta^2}^\top \nabla_{23} L^2 {\theta^3} + {\theta^1}^\top \nabla_1 L^2 + {\theta^3}^\top \nabla_3 L^2 + \frac{1}{2\eta}{\theta^2}^\top {\theta^2}
        } \\
        \min_{\theta^3 \in \mathbb{R}^{d_3}}{
            {\theta^3}^\top \nabla_3 L^3 + {\theta^3}^\top \nabla_{31} L^3 {\theta^1} + {\theta^3}^\top \nabla_{32} L^3 {\theta^2} + {\theta^1}^\top \nabla_1 L^3 + {\theta^2}^\top \nabla_2 L^3 + \frac{1}{2\eta}{\theta^3}^\top {\theta^3}
        } 
    \end{gathered}
\label{eq:pcgd_update}
\end{equation}
in the special case of three players.
The first term (${\theta^1}^\top \nabla_1 L^1$) is an update strictly in the direction of decreasing cost, or the self-minimization term. The second group of terms such as ${\theta^1}^\top \nabla_{12} L^1 {\theta^2} + {\theta^1}^\top \nabla_{13} L^1 {\theta^3}$ correspond to choosing an update with respect to what other agents' choices might be, which we refer to as the interaction term. The third terms such as ${\theta^2}^\top \nabla_2 L^1 + {\theta^3}^\top \nabla_3 L^1$ estimate the impact of other agents' update on one's objective. The fourth term, e.g. $\frac{1}{2\eta}{\theta^1}^\top {\theta^1}$ is the step size regularization. 
Contrary to the full multilinear approximation, the multilinear polymatrix approximation can be solved by using highly efficient algorithms for numerical linear algebra, such as Krylov subspace methods.

\paragraph{Practical implementation.} 
The update of Algorithm~\ref{alg:pcgd} is given in terms of the mixed Hessian, which scales quadratically in size with respect to the number of agent parameters.
However, matrix-vector products $v \mapsto \nabla_{ij} L^i \cdot v$ can be computed as efficiently as the value of the loss function by combining forward and reverse mode automatic differentiation to compute Hessian vector products. 
Without access to mixed mode automatic differentiation, we employ the ``double backprop trick'' that computes matrix-vector products as $\nabla_{ij} L^i \cdot v = \nabla_{i}(\nabla_{j} L^i \cdot v)$. 
The quantity on the right can be evaluated by applying reverse-mode automatic differentiation twice.
Using either of these methods, the computational cost of evaluating matrix-vector products in the mixed Hessian is, up to a small constant factor, no more expensive than the evaluation of the loss function.

We combine the fast matrix-vector multiplication with a Krylov subspace method for the solution of linear systems to compute the update of Algorithm~\ref{alg:pcgd}.
In our experiments, we use the conjugate gradient algorithm~\citep{o1980block} to solve the system by solving the positive semi-definite system $M^\top M x = M^\top y$, where $M$ is the block-matrix $(I+\eta H_o)$ we wish to invert. We terminate conjugate gradient after a relative decrease of the residual is achieved ($||M^\top Mx - M^\top y|| \leq \epsilon ||M^\top y||$ for some threshold $\epsilon$). 
To decrease the number of iterations needed, we use the solution of the previous updates as an initial guess for the next iteration of conjugate gradient.
The number of iterations needed by the inner loop is highly problem-dependent.
However, we observe that the updates of PCGD are often not much more expensive than SimGD in practice.
In applications to reinforcement learning, the sampling cost often masks the computational overhead of PCGD.

\section{Theoretical Results}
\label{sec:theory}
We now provide theoretical results on the convergence of PCGD. 
% In this section, we provide a theoretical study of the convergence properties of PCGD in the vicinity of local Nash equilibria of the game. 
\begin{definition}[Local Nash equilibrium]
    We call a tuple of strategies $\bar{\theta} = \left(\bar{\theta}^1, \ldots, \bar{\theta}^n \right)$ a \emph{local Nash equilibrium} if for each $1 \leq i \leq n$ there exists an $\epsilon_i > 0$ such that 
    \begin{equation}
        \|\theta^i - \bar{\theta}^i\| < \epsilon_i \Rightarrow L^i\left(\bar{\theta}^1, \ldots, \theta^i, \ldots, \bar{\theta}^n\right) \geq L^i\left(\bar{\theta}^1, \ldots, \bar{\theta}^i, \ldots, \bar{\theta}^n\right).
    \end{equation}
\end{definition}

We use the notation introduced in Section~\ref{sec:pcgd} and write the symmetric and nonsymmetric parts of the Hessian matrix as $S \coloneqq (H + H^{\top}) / 2$ and $A = (H - H^{\top}) / 2$.
The matrices $S$ and $A$ can be seen as capturing the pair-wise collaborative and competitive parts of the game, respectively.
To see this, we fix all agents but those with indices $1 \leq i, j \leq n$ and obtain the game
\begin{equation}
    \min \limits_{\theta^i} L^i(\theta^1, \ldots, \theta^i, \ldots \theta^j \ldots, \theta^n),
    \quad
    \min \limits_{\theta^j} L^j(\theta^1, \ldots, \theta^i, \ldots \theta^j \ldots, \theta^n).
\end{equation}
Writing now $s = (L^i + L^j) / 2$ and $a = (L^i - L^j) / 2$, we can write this game as the sum of a strictly collaborative potential game and a strictly competitive zero-sum game
\begin{align}
    &\min \limits_{\theta^i} s(\theta^1, \ldots, \theta^i, \ldots \theta^j \ldots, \theta^n) + a(\theta^1, \ldots, \theta^i, \ldots \theta^j \ldots, \theta^n), \\
    &\min \limits_{\theta^j} s(\theta^1, \ldots, \theta^i, \ldots \theta^j \ldots, \theta^n) - a(\theta^1, \ldots, \theta^i, \ldots \theta^j \ldots, \theta^n).
\end{align}
The Hessian of the potential game in this decomposition is given by the restriction of $S$ to the strategy spaces of the $i$-th and $j$-th. 
The off-diagonal blocks of the Hessian, on the other hand, are given by the corresponding blocks of $A$.
We now state our main theorem.

\begin{theorem}
    \label{thm:convergence}
    Assume that each of the $\left\{L^i\right\}_{1 \leq i \leq n}$ is twice continuously differentiable and let $\bar{\theta}$ be a local Nash equilibrium for which the game Hessian $H(\bar{\theta})$ is invertible.
    Then, for all $0 < \eta < \frac{1}{4\|S\|}$, there exists an open neighborhood of $\bar{\theta}$ such that PCGD converges at an exponential rate to $\bar{\theta}$.
\end{theorem}
The authors of \cite{schafer2019competitive} show that in the case of two-player, zero-sum games, the convergence of CGD is robust to strong agent-interaction as described by large off-diagonal blocks of the game Hessian, without lowering the step size.
In stark contrast, methods such as extragradient \citep{korpelevich1977extragradient}, consensus optimization \citep{mescheder2017numerics}, or symplectic gradient adjustment \citep{balduzzi2018mechanics} have to decrease the step size to counter strong agent interactions.
The results presented here extend the results of \cite{schafer2019competitive} to general-sum games with arbitrary numbers of players and establish convergence of PCGD that is fully robust to the strength of the competitive interactions given by $A$, without additional step size adaptation. 

Having the ability to choose larger step sizes allows our algorithm to achieve faster convergence. 
The example below illustrates that in some games, e.g., multilinear games with pairwise zero-sum property, the step size can be arbitrarily large while PCGD still guarantees convergence.
\begin{example}
Consider a four-player multilinear game with pairwise zero-sum interactions: $L^1 = \theta^1 \theta^2 + \theta^1 \theta^3 + \theta^1 \theta^4, L^2 = -\theta^1 \theta^2+\theta^2 \theta^3+\theta^2\theta^4, L^3 = -\theta^1 \theta^3-\theta^2\theta^3+\theta^3\theta^4, L^4=-\theta^1\theta^4-\theta^2\theta^4-\theta^3\theta^4$. The simultaneous gradient is,
$$\xi = (\theta^2+\theta^3+\theta^4, -\theta^1+\theta^3+\theta^4, -\theta^1-\theta^2+\theta^4, -\theta^1-\theta^2-\theta^3)\,,$$
and game Hessian is 
$$H = \begin{bmatrix}
0 & 1 & 1 & 1\\
-1 & 0 & 1 & 1\\
-1 & -1 & 0 & 1\\
-1 & -1 & -1 &0
\end{bmatrix} \,.$$

The origin is a stable fixed point since $S = 0 \succeq 0$ and $H$ is invertible. We have $\|S\| = 0$, so PCGD converges locally (in fact globally) to the origin for \textbf{all} $0 < \eta < \frac{1}{2||S||} = \infty$!
\end{example}

We note that the step size of PCGD still needs to be adapted to the magnitude of $S$. 
In the single-agent case, $S$ is the curvature of the objective which limits the size of step size for gradient descent.
Thus it is expected that $S$ limits the step size of multi-agent generalizations of gradient descent. 

We also remark that while we proved local convergence of PCGD to local Nash equilibria, an attractive point of the PCGD dynamics need not be a local Nash equilibrium if we consider general non-convex loss functions. 
A similar phenomenon was observed by \citet{mazumdar2018convergence} in the context of simultaneous gradient descent and by \citet{daskalakis2018limit} in the context of optimistic gradient descent. 
These observations motivated some authors, such as \citet{mazumdar2019finding}, to search for algorithms that \emph{only} converge to local Nash equilibria.
Others, such as \citet{schafer2020competitive} instead question the importance of local Nash equilibria as a solution concept in non-convex multi-agent games.
A complete characterization of the landscape of attractors of PCGD is an interesting direction for future work.

\section{Instantiation to Multi-Agent Reinforcement Learning} 
\label{sec:pcgd_rl}
%Reinforcement learning techniques have been widely used in the context of learning in multiagent systems.
A fundamental challenge in multi-agent reinforcement learning is to design efficient optimization methods with provable convergence and stability guarantees. In this section, we instantiate how the proposed PCCD can be used for policy optimization in multiagent reinforcement learning. 
In particular, when deriving the policy updates, PCGD captures not only the agent's own reward function, but also the effects of all other agents' policy updates. 

%\todo{We also point out that while we have proved local convergence to }
\paragraph{Multiagent Reinforcement Learning (MARL).}  An $n$-player Markov decision process is defined as the tuple $\langle S, A^{1}, A^{2}, \cdots, A^{n}, r, {\gamma}, P\rangle$ where $S$ is the state space, $A^{i}, i=1,...,n$ is the action space of player $i$. $\{r_i\}_{i=1}^{n}: S \times A^1 \times ... \times A^n \rightarrow \mathbb{R}^n$ denote reward where $r_i$ is a bounded reward function for player $i$; $\gamma \in (0, 1]$ is the discount
factor and $P: S \times A^1 \times ... \times A^n \rightarrow \Delta^{S}$ maps the state-action pairs to a probability distribution over next states. We use $\Delta^{S}$ to denote the $|S|-1$ simplex. The goal of reinforcement learning is to learn a policy $\pi(\theta^i): S \rightarrow \Delta^{A_i}$ which maps a state and action to
a probability of executing the action from the given state in order to maximize the expected return. In the n-player setting, the expected reward of each player is defined as,
\begin{equation}\label{eq:rl_reward}
    \forall i, J^i(\theta^1, \theta^2, ..., \theta^n) = E[R_i(\tau)] = \int_{\tau} f(\tau; \theta^1, ..., \theta^n) R_i(\tau) d\tau\,,
\end{equation}
which is a function associated with all players' policy parameters $(\theta^1, ..., \theta^n)$. $f(\tau; \theta^1, ..., \theta^n) = p(s_0) \prod_{t=0}^{T-1} \pi(a_t^1|s_t; \theta^1)...\pi(a_t^n|s_t; \theta^n) P(s_{t+1}|s_t, a_t^{1}, ..., a_t^{n})$ denotes the probability distribution
of the trajectory $\tau$ and $R_i(\tau) = \sum_{t=0}^{T} \gamma^t r_i(s_t, a_t)$ is the accumulated trajectory reward.

Each agent aims to maximize its expected reward in \eqref{eq:rl_reward}, i.e.,
\begin{equation}\label{eq:multiagent_copo}
\forall i \in \{1, ..., n\}\,, \max_{\theta^i \in \mathbb{R}^{d_i}} J^i(\theta^1, ..., \theta^n)
\end{equation}

\paragraph{PCGD for MARL training.} \Eqref{eq:multiagent_copo} is an instantiation of the multi-agent competitive optimization setting defined in Section~\ref{sec:pcgd}. Thus, we can use PCGD for the policy update,
\begin{equation}\label{eq:pcgd_pg}
    \theta_{k+1} =  \theta_{k} +\eta (I + \eta H_o)^{-1}\xi
\end{equation}
%\alistair{shouldn't the update be with a PLUS sign since $J = -L$ are objectives, not losses?}
where $\xi$ is the simultaneous gradient $\xi = \left(\nabla_{\theta_1} J^1, ..., \nabla_{\theta_n} J^n\right)^\top$, and $H_o$ if the off-diagonal part of the game Hessian. 
In order to compute gradients and Hessian-vector-products, we rely on policy gradient theorems that allow us express the gradients and hessian-vector-products as expectations that can be approximated using sampled trajectories.
%To approximate the gradient and Hessian terms, one can use sampled trajectories such as \texttt{REINFORCE}~\citep{williams1992simple}, or learn a value function to approximate the expected accumulated reward such as actor-critic methods~\citep{konda2000actor}. 
The simplest approach to computing Hessian-vector products, which is applicable to \emph{any} reward function $R$, amounts to simply applying the policy gradient theorem twice, resulting in the expression, where 
\begin{subequations}
\begin{align}
\nabla_{\theta_i} J^i &= \mathbb{E}_{\tau} \left[ \sum_{t=0}^{T-1} \nabla_{\theta_i} \log \pi_{\theta^i}(a_t^i|s_t) R(s_t, a_t^1, ..., a_t^n) \right]\\
\nabla_{\theta_i, \theta_j} J^i &= E_{\tau}\left[\sum_{t=0}^{T-1} \nabla_{\theta_i} \log \pi_{\theta^i}(a_t^i|s_t) \nabla_{\theta_j} \log \pi_{\theta^j}(a_t^j|s_t) R(s_t, a_t^1, ..., a_t^n)\right]
\end{align}
\end{subequations}

The downside of this approach, which can be seen as a competitive version of \texttt{REINFORCE}~\citep{williams1992simple}, is that it does not exploit the Markov structure of the problem, resulting in poor sample efficiency.
To overcome this problem, \citet{prajapat2020competitive} introduce a more involved competitive policy gradient theorem that expresses the mixed derivatives in terms of the Q-function. 
% At each training epoch, agents deploys $\pi(\theta^1), ..., \pi(\theta^n)$ to get a batch of N trajectories, consisting the state, action, and reward at each trajectory step. 
% Then, we exploit them to estimate $Q$ function (state-action value), gradient $\nabla_{\theta_i} J^i$ and Hessian $\nabla_{\theta_i, \theta_j} J^i$ as follows:
\begin{equation}
    \nabla_{\theta_i} J^i = \mathbb{E}_{\tau} \left[ 
    \sum_{t=0}^{T-1} \gamma^t \nabla_{\theta_i} \log \pi_{\theta^i}(a_t^i|s_t) Q(s_t, a_t^1, ..., a_t^n)
\right]
\end{equation}
\begin{equation}
\begin{split}
    \nabla_{\theta_i, \theta_j} J^i = E_{\tau}\left[\sum_{t=0}^{T-1} \gamma^t \nabla_{\theta_i} \log \pi_{\theta^i}(a_t^i|s_t) \nabla_{\theta_j} \log \pi_{\theta^j}(a_t^j|s_t) Q(s_t, a_t^1, ..., a_t^n)\right] \\
 + E_{\tau}\left[
\sum_{t=0}^{T-1} \gamma^t \nabla_{\theta_i} \log \left(\prod^{t-1}_{l=0}\pi_{\theta^i}(a_l^i|s_l)\right) \nabla_{\theta_j} \log \pi_{\theta^j}(a_t^j|s_t) Q(s_t, a_t^1, ..., a_t^n) \right] \\
 + E_{\tau}\left[
\sum_{t=0}^{T-1} \gamma^t \nabla_{\theta_i} \log \pi_{\theta^i}(a_t^i|s_t) \nabla_{\theta_j} \log \left(\prod^{t-1}_{l=0}\pi_{\theta^j}(a_l^j|s_l)\right) Q(s_t, a_t^1, ..., a_t^n)
\right]
\end{split}
\end{equation}

This formulation allows us to use advanced, actor-critic-type approaches \citep{konda2000actor} to improve the sample efficiency.
%\todo{some typos below, also mention that we are replacing expectation by batch}
In our implementation, we use generalized advantage estimation (GAE)~\cite{schulman2015high} $A(s, a^1, ..., a^n) = Q(s, a^1, ..., a^n) - V(s)$ in place of the $Q$ function to calculate the gradient and Hessian terms. After sampling a batch of states, actions, and rewards from the replay buffer, we construct two pseudo-objectives, one for the first derivative terms and one for the mixed Hessian terms required for the PCGD update. Taking the gradient with respect to a single player of the one pseudo-objective gives us the desired single agent ($\xi$) terms for each player in the PCGD update, while taking the mixed derivative of the second pseudo-objective with respect to two different players using auto-differentiation gives us the desired polymatrix or interaction terms ($H_o)$.

\section{Numerical Experiments}
\label{sec:experiment}
% \todo{Jeffrey: Baseline compasions to SGD, SGA, and Optimistic gradient descent on all three games}
\paragraph{Snake and Markov Soccer.} We demonstrate our algorithm through numerical experiments on two games: Four-player Snake and four-player Markov Soccer. In the Snake game, four different snakes compete in a fixed space to box each other out, to consume a number of randomly-placed fruits, and to avoid colliding with each other or the walls. At each step of the game, each snake observes the 20-by-20 space around its head and decides whether to continue moving forward, to turn left, or to turn right. Agents are rewarded with $+1$ from consuming a fruit and $-1$ for a collision with either another snake's body or a wall. Such a collision removes the snake from the game and the last snake alive receives another reward of $+1$.
If a snake is removed from the game for colliding with another snake, that other snake receives a reward of $+20$ for ``capturing'' the first snake.
We build on the implementation of Snake as used by \citet{marlenv2021}.
Four-player Markov Soccer is an extension of the two-player variant proposed by \citet{conf/icml/Littman94} and consists of four players and a ball that are randomly initialized in a 8-by-8 field. Each agent is assigned to a goal and looks to pick up the ball or steal it from opposing players and score it in an opponent's goal. The goal-scoring player is given $+1$, where all other players are penalized with $-1$; in the case where a player own-goals, that player is penalized with -1, and all other players are given +0.

% \begin{figure}[ht]
%     \begin{minipage}{0.42\textwidth}
%         \centering
%         \includegraphics[width=0.7\linewidth]{figures/snake_game.png}
%         \caption{A game of \emph{snake} between involving four agents. The goal of the agents is to eat the fruit (in red) and force other snakes to run into them to ``kill'' them.}
%         \label{fig:snake_game_example}
%     \end{minipage}\hfill
%     \begin{minipage}{0.42\textwidth}
%         \centering
%         \includegraphics[width=0.7\linewidth]{figures/markov_soccer.png}
%         \caption{In four-player Markov Soccer, each agent (colored circle) tries to obtain the ball (in black) and tries to place it in one of the other agents' goals.}
%         \label{fig:markov_soccer_game_example}
%     \end{minipage}\hfill
% \end{figure}

\begin{figure}[ht]
    \begin{minipage}[t]{0.33\textwidth}
        \centering
        \includegraphics[width=1\linewidth]{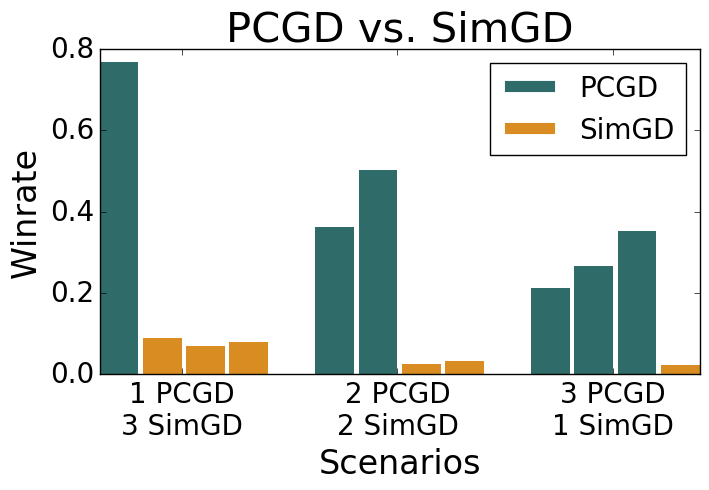}
    \end{minipage}\hfill
    \begin{minipage}[t]{0.33\textwidth}
        \centering
        \includegraphics[width=1\linewidth]{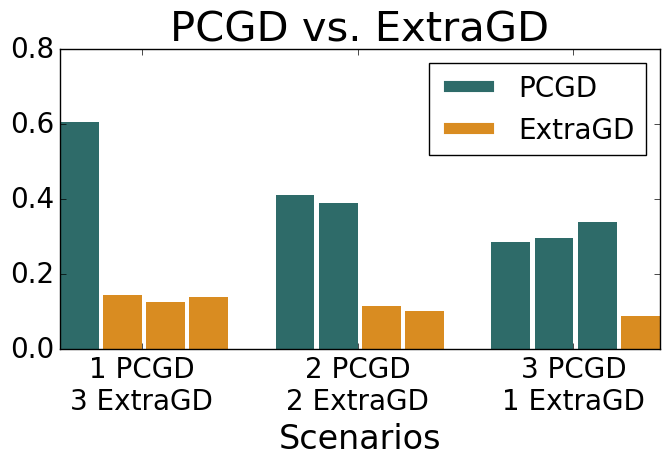}
    \end{minipage}\hfill
    \begin{minipage}[t]{0.33\textwidth}
        \centering
        \includegraphics[width=1\linewidth]{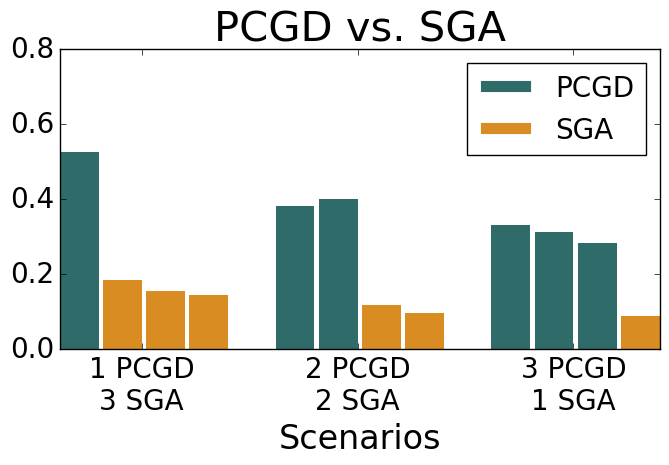}
    \end{minipage}\hfill
\caption{\textbf{Snake game, win rate:} We plot the rate at which the different agents win the game (by achieving the highest reward).}
\label{fig:snake_game}
\end{figure}

\begin{figure}[ht]
    \begin{minipage}[t]{0.33\textwidth}
        \centering
        \includegraphics[width=1\linewidth]{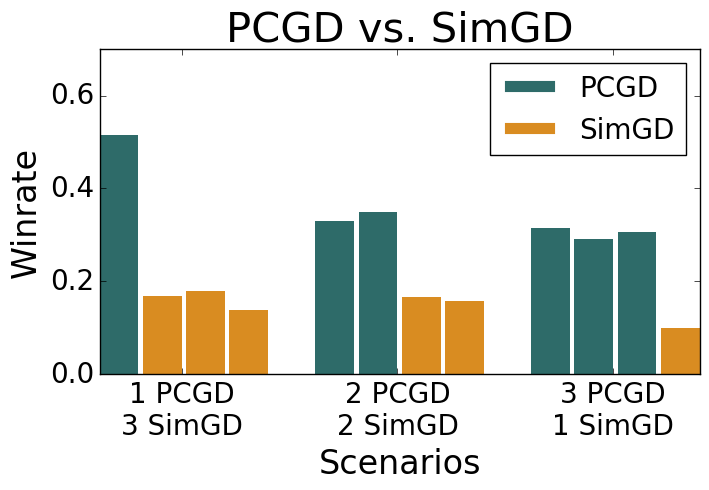}
    \end{minipage}\hfill
    \begin{minipage}[t]{0.33\textwidth}
        \centering
        \includegraphics[width=1\linewidth]{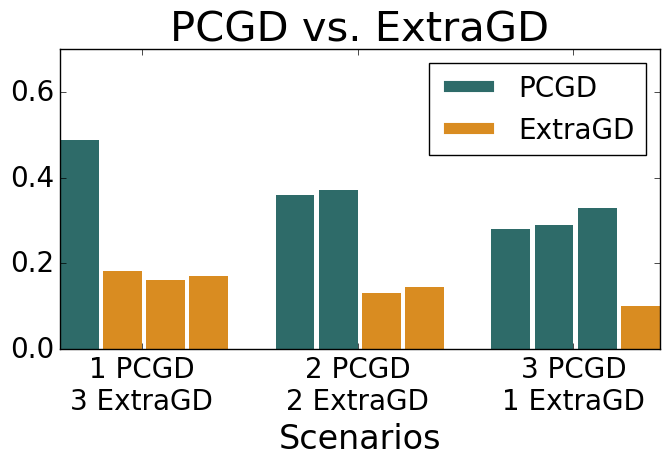}
    \end{minipage}\hfill
    \begin{minipage}[t]{0.33\textwidth}
        \centering
        \includegraphics[width=1\linewidth]{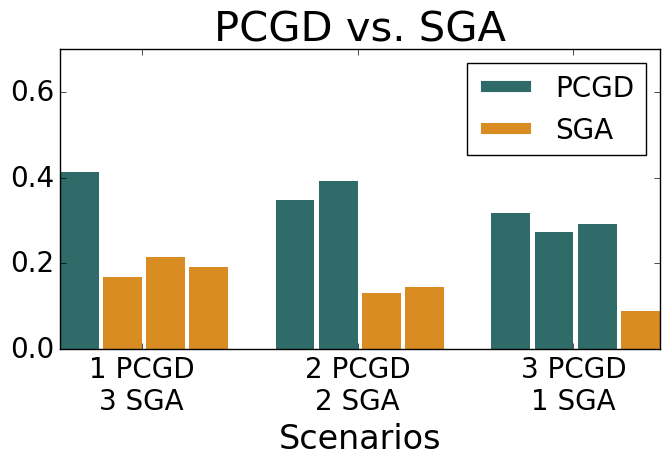}
    \end{minipage}\hfill
\caption{\textbf{Markov Soccer, win rate:} We plot the rate at which the different agents win the game (by achieving the highest reward).}
\label{fig:soccer_game}
\end{figure}

For both of these games, we take four policies and train them together with the same algorithm, once under PCGD and once under three other methods: SimGD, Extragradient, and Sympletic Gradient Ascent (SGA) (\cite{letcher2019differentiable}). We then compare the policies generated by PCGD and one of the competing methods in three scenarios: playing 1 PCGD-trained agent against 3 competitor-trained agents; 2 PCGD-trained agents against 2 competitor-trained agents; and 3 PCGD-trained agents against 1 competitor-trained agent. 
%For Snake, we compare learned strategies by examining the frequency of winning in each major objective: being the last alive, having the most kills, and consuming the most fruit. 
As illustrated in Figures~\ref{fig:snake_game} and \ref{fig:soccer_game}, we find that the PCGD trained agents show significantly higher winrate, even when facing major distributional shift in the case where a single PCGD agent plays against three agents trained using the same competing method.
%Similar results are obtained in the case 
%We also observe that the PCGD trained policy is more than nine times more likely to be the last alive and is more than 2.5 times more likely to have the most number of snake kills as compared to SimGD trained counter parts. Likewise, for Markov Soccer, we examine win-rate and number of steals in the same three scenarios. 
%In a similar spirit to the results on Snake we observe that when the PCGD agents have significantly higher rates of winning the game and stealing the ball from their opponents, in all three mixing scenarios. Likewise, PCGD shows a significant winrate over the Extragradient and SGA methods. \todo{[TODO: Include more description about other baseline results]}

Details about the implementation and choice of hyperparameters of the Snake game and Markov soccer are provided in the appendix.

\paragraph{Electricity Market Simulation.} We also demonstrate the utility of PCGD for the simulation of strategic agent behavior in social-economic games. In particular, we consider an electricity market game where multiple generators compete for electricity supply and profit maximization following the setup in~\citep{yu2007modeling}.
%in electricity market simulation following the real-world electricity market operating rules~\cite{yu2007modeling}. 
%Learning to bid in the electricity market is crucial for small electricity suppliers (e.g., investor-owned renewable generators) to recover their investment cost by bidding strategically to maximize their profit~\cite{boukas2020deep,deepmind2019}. On the other hand, simulations about generators' gaming behavior can provide insights for market operator to refine mechanism design and enhance market efficiency. 
% In particular, we demonstrate that PCGD provides an efficient computation tool for the market operator to fast simulating potential gaming behavior of agents, and thus provides insights for better mechanism design. 
In the following, we explain how electricity market is modeled as a MARL game, by defining the state, action, and reward.
% \begin{figure*}[ht]
%     \centering
%     \includegraphics[width=0.85\textwidth]{figures/electricity_game.png}
%     \caption{MARL game for electricity suppliers in the electricity market.}
%     \label{fig:cournot_results}
% \end{figure*}
%While here we focus on electricity market, our approach can be used for agent-based simulations in other social-economical systems such as transportation network~\cite{zhang2019cityflow}.

%In this work, we follow the formulation in~\citet{yu2007modeling} where the electricity market is modeled as a multi-agent reinforcement learning game. It includes three types of agents, that are generation agents (e.g., generators), demand (e.g., utilities and load service entities) and a market operator. At each time step\footnote{Depending on the market clearance scale, each time step represents a day, an hour or 15 minutes.}, generators submit supply offers and demand entities submit demand bids for the next time step to the market operator. The market operator runs the market clearing algorithm to match the generation to demand at the lowest cost, that determines generation dispatch schedule and electricity price at each location. The interactions between the generators, demand and market operator is shown in Fig~\ref{fig:cournot_results}. 

\emph{1) State and Actions.} We follow the formulation in~\citep{yu2007modeling} where the system state is defined by $\mathbf{s}_t = ((d_{t, k})_{k \in \mathcal{N}})$ that is the electricity demand at time $t$ across all nodes $\mathcal{N} = \{1, ..., N\}$. 
%In our simulation, we focus on the repeated single-stage game where the demand is inelastic to price~\cite{kirschen2018fundamentals}. 
%There are two demand status at each node, $d_{k}=0$ (normal demand level) and $d_{k} = 1$ (reduced demand level). If the electricity price is higher than a critical value, the demand at next time step will reduce to $d_{k} = 1$, otherwise it stays at the normal level. 
% One can also consider more complicated price-demand, and thus state-transition functions. But as ~\cite{kirschen2018fundamentals}. 
For generator $i$, its action $\mathbf{a}_{i, t} = (c_{i, t}, p_{i, t}^{max})$ includes two parts: the first part is the per unit production cost $c_{i, t} \in \mathbb{R}^{\geq 0}$ (\$/MWh) and the second part is the supply capacity $p_{i, t}^{max} \in \mathbb{R}^{\geq 0}$ (MW). We assume that agents are \emph{strategic} by physically or economically withholding supply from the market.

%honest about their production cost (as cheating can be easily detected \cite{?}) but can cheat on its maximum production capacity (by bidding lower than its true capacity limit).
%Let $I_k = \{i: \text{generator $i$ located at node $k$}\}$ denote generators at node $k$ and $I$ is the set of all generators. 

\emph{2) Reward:} The goal of each generator is to maximize its profit, which can be written as,
$$r_{i, t}(\mathbf{s}_t, \mathbf{a}_t) = p_{i, t}^{*} \cdot \mathrm{LMP}_{k, t} - c_{i, t} \cdot p_{i, t}^{*}\,, \forall i \in \mathcal{I}_k$$
where the first part represents market payment for generating $p_{i, t}^{*}$ unit of electricity and the second part reflects the production cost.
In the above reward function, the scheduled supply quantity $p_{i, t}^{*}$, and the unit electricity price $\mathrm{LMP}_{k, t}$ at location $k$ (suppose generator $i$ locates in node $k$), are both decided by the system operator via solving the following optimization problem.
%Basically the market operator solves an optimization problem to decide the scheduled capacity and price for each generator,
% \begin{align}
%     \min_{\mathbf{p}, \mathbf{f}} \quad & \mathbf{c}^T \mathbf{p}  \\
%     \text{s.t.} \quad & 0 \leq \mathbf{p} \leq \mathbf{p}^{max} \\
%     & -\underline{\mathbf{f}} \leq K\mathbf{f} \leq \overline{\mathbf{f}}\\
%     & \mathbf{p} + A \mathbf{f} = \mathbf{d} 
% \end{align}
% \begin{subequations}\label{eq:lmp}
% \begin{align}
%     \min_{p_i, \theta_k} \quad & \sum c_i p_i  \label{eq:min_cost}\\
%     \text{s.t.} \quad & \sum_{i \in I_k} p_i - D_k = \sum_{m=1, m \neq k}^{N} [1/{x_{km}}(\theta_k-\theta_m)]\,, \text{for all node $k$} \label{eq:power_balance}\\
%     & F_{km}^{min} \leq 1/{x_{km}} [\theta_k-\theta_m] \leq F_{km}^{max} \,, \text{for all line $km$} \label{eq:flow_limit}\\
%     & 0 \leq p_i \leq p_i^{max}, \text{for all generator $i$}  \label{eq:gen_limit}
% \end{align}
% \end{subequations}
\begin{subequations}\label{eq:lmp}
\begin{align}
    \min_{\bm{p}_t} \quad & \sum_{t \in \mathcal{T}} \sum_{i \in \mathcal{N}}  c_{i, t} p_{i, t}  \label{eq:min_cost}\\
    \text{s.t.} \quad & g(\bm{s}_t,\bm{a}_t, \bm{p}_t)=\bm{0}; \label{equ:g}\\
    & \quad h(\bm{s}_t,\bm{a}_t, \bm{p}_t)\leq \bm{0}; \label{equ:h}
\end{align}
\end{subequations}
%where the endogenous variables are $p_i$ (generator $i$'s supply quantity) and $\theta_k$ (voltage angle at node $k$), and the exogenous variables are $D_k$ (demand at at node $k$). 
The objective function~\ref{eq:min_cost} aims to minimize the total system generation cost by deciding the power supply from each generator $p_{i, t}$ over horizon $\mathcal{T}$ (e.g., 24 hours). The constraint in \eqref{equ:g} collects the power grid equality constraints such as supply-demand balance and power flow equations; the constraint in \eqref{equ:h} encodes physical limits, such as maximum generation capacity, i.e., $0 \leq p_{i, t} \leq p_{i, t}^{max}$ and line flow limits. We refer the reader to~\citep{yu2007modeling} for more details.

%the power balance constraint at each node~\ref{eq:power_balance}, the branch power flow limit~\ref{eq:flow_limit} and each generator's capacity constraint~\ref{eq:gen_limit}. In Eq~\ref{eq:lmp}, optimization variable $p_i$ and $\theta_k$ represent generator $i$'s supply quantity and voltage angle at node $k$; $I_k$ denotes all generators located at node $k$; $x_{km}$ is the transmission line reactance, $F_{km}^{min}$, $F_{km}^{max}$ are the minimum and maximum line power flow capacity. about the electricity market optimization problem. 

% Once the market optimization~\eqref{eq:lmp} is solved, it decides the production level of each generator $p_i^{*}$, and the unit electricity price paid to the generator which depends on generator's location (e.g., suppose $i$ locates in node $k$). Therefore, generator $i$'s reward can be written as,
% $$r_i(\mathbf{s}, \mathbf{a}) = p_i^{*} \cdot LMP_{k} - c_i \cdot p_i^{*}\,, \forall i \in \mathcal{I}_k$$
% where the first represents market payment and the second part reflect its production cost.

\emph{3) Learning algorithm:}
We use SimGD, SGA, ExtraGradient, and PCGD for simulating strategic behavior in the electricity market game. 
While all optimization methods achieve comparable reward in this game, we observe that PCGD converges significantly faster than both SimGD and Extragradient, and slightly faster than SGA.
Despite the need of solving a linear system at each step of PCGD, this holds true when measuring cost in terms of either the number of iterations or total wall-clock time. 
This observation is explained by the fact that most of the iteration cost occurs when sampling from the policy, which has to be done only once per iteration. 
% We also find that the cost per iteration of CGD is only slightly higher than that of SimGD, by using conjugate gradient. This can be partially explained by most of the iteration cost occurs during sampling, masking the actual overhead of PCGD vs SimGD. Thus, the total training time for PCGD is significantly faster than SimGD. 

% \todo{TODO: Generate graphs for additional baselines for electricity market.}
% \todo{[TODO: Discuss hypothesis on why electriciy market PCGD and other methods do bet ter than SimGD]}

\begin{figure}[h]
    \centering
    \includegraphics[width=0.32\linewidth]{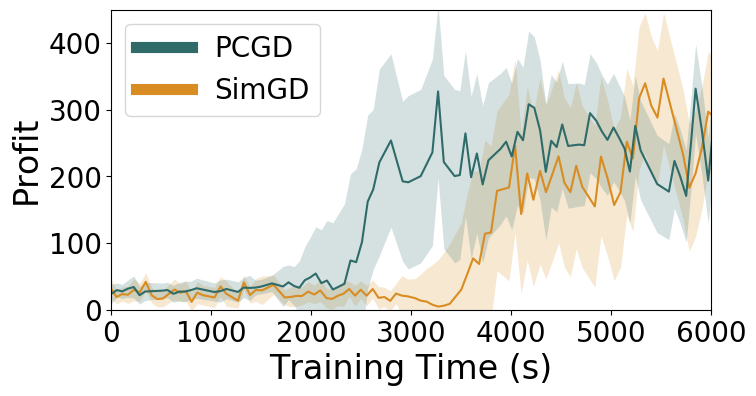}
    \includegraphics[width=0.32\linewidth]{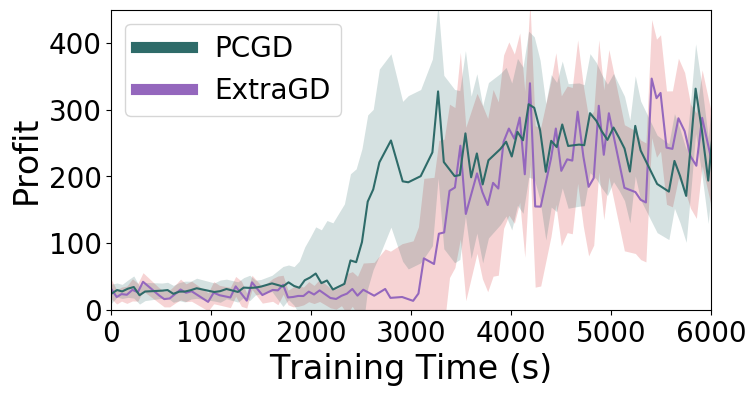}
    \includegraphics[width=0.32\linewidth]{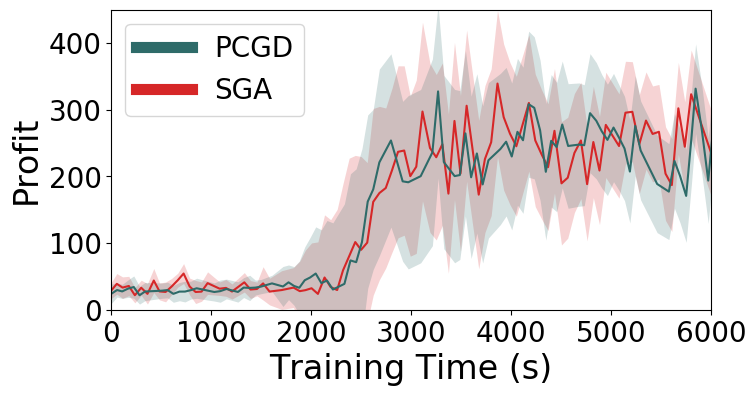}
    % \caption{\textbf{Electricity Market, accumulated profit vs. training time for the first of the three learning generators:} We plot the accumulated profit versus number of training steps (in the first row) and the profit versus training time (in the second row) for three learning generators.}
    \caption{\textbf{Electricity Market, accumulated profit vs. training wall clock time for one of three learning generators:} We plot the accumulated profit versus wall clock time elapsed for PCGD versus each of the three other methods: SimGD, Extragradient, and SGA.}
    \label{fig:elec_market_converge_wall_clock}
\end{figure}

\begin{figure}[h]
    \centering
    \includegraphics[width=0.32\linewidth]{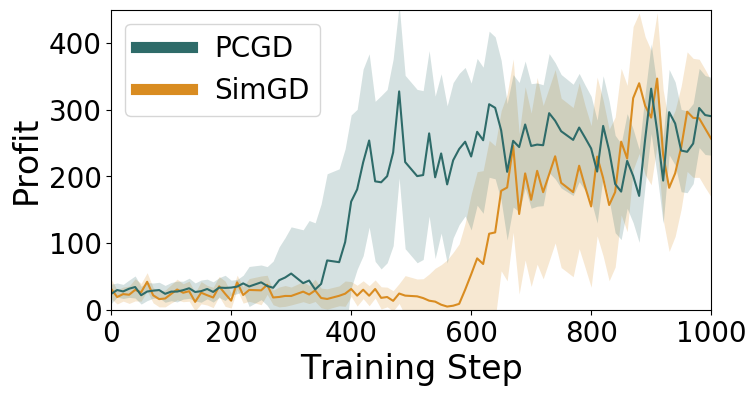}
    \includegraphics[width=0.32\linewidth]{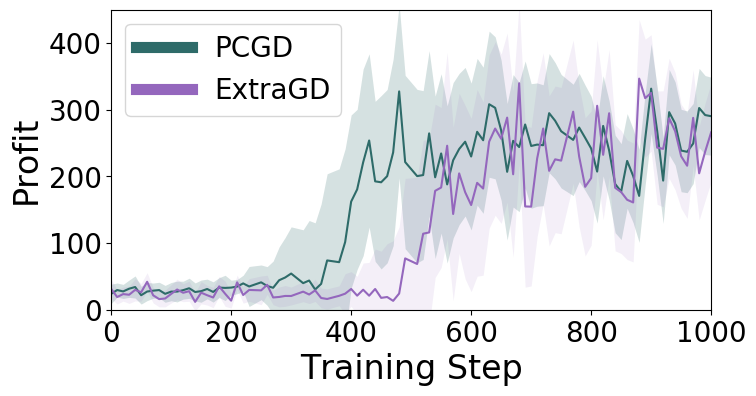}
    \includegraphics[width=0.32\linewidth]{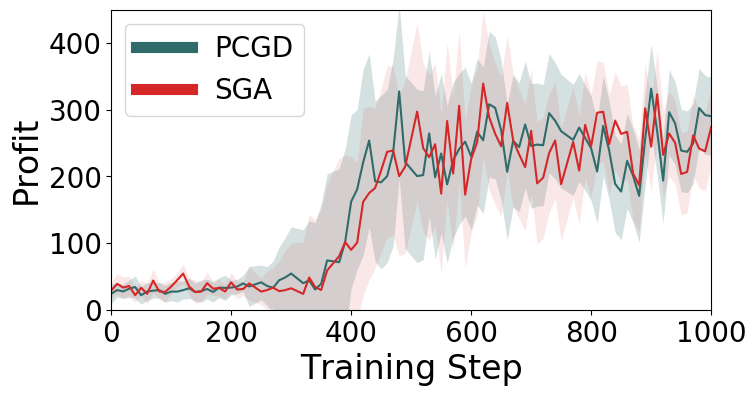}
    % \caption{\textbf{Electricity Market, accumulated profit vs. training time for the first of the three learning generators:} We plot the accumulated profit versus number of training steps (in the first row) and the profit versus training time (in the second row) for three learning generators.}
    \caption{\textbf{Electricity Market, accumulated profit vs. number of training steps for the one of three learning generators:} We plot the accumulated profit versus number of training steps for PCGD versus each of the three other methods: SimGD, Extragradient, and SGA.}
    \label{fig:elec_market_converge}
\end{figure}

\section{Conclusion}
In this work, we present a method for solving general sum competitive optimization of an arbitrary number of agents, named polymatrix competitive gradient descent (PCGD).
We prove that PCGD converges locally in the vicinity of a local Nash equilibrium without having to reduce the step size to account for strong competitive player-interaction. 
We then apply our method towards a set of multi-agent reinforcement learning games, and show that PCGD produces significantly more performant policies on Markov Soccer and Snake, while achieving reduced training cost on an energy market game.
PCGD provides an efficient and reliable computational tool for solving multi-agent optimization problems, which can be used for agent-based simulation of social-economical games, and also provides a powerful paradigm for new machine learning algorithm design that involves multi-agent and diverse objectives. 
A number of open questions remain for future work. While our present results show the practical usefulness of the polymatrix approximation, it might be inadequate for games with strong higher-order-interactions that are not captured by the polymatrix approximation and instead require the full multilinear approximation.
Another open question is to better characterize the properties of solutions obtained by PCGD in collaborative games.
PCGD agents \emph{beat} agents trained by SimGD in our experiments, but in some applications we want agents to maximize a notion of social good. It is not clear what effect PCGD has on this objective.

\section*{Acknowledgements}
AA is supported in part by the Bren endowed chair, Microsoft, Google, and Adobe faculty fellowships.
FS gratefully acknowledge support by the Air Force Office of Scientific Research under award number  FA9550-18-1-0271 (Games for Computation and Learning) and the
Ronald and Maxine Linde Institute of Economic and Management Sciences at Caltech.

\bibliographystyle{unsrtnat}
\bibliography{refs}

\appendix
\section{Appendix}
\appendix

\section{Proof of theoretical results}

%\subsection{Proof of theoretical results}
A popular tool to establish local convergence of iterative solvers is \emph{Ostrowski's theorem}. 
We will use this classical theorem in the form of \cite[Theorem 8]{letcher2019differentiable}, which is in turn an adaptation of \cite[10.1.3)]{ortega2000iterative}.
\begin{theorem}[Ostrowski]
    \label{thm:ostrowski}
    Let $F: \Omega \rightarrow \mathbb{R}^{d}$ be a continuously differentiable map on an open subset $\Omega \subset \mathbb{R}^{d}$, and assume $\mathbf{w}^{*} \in \Omega$ is a fixed point of $F$. 
    If all eigenvalues of $\nabla F(\mathbf{w}^*)$ are strictly in the unit circle of $\mathbb{C}$, then there is an open neighbourhood $U$ of $\mathbf{w}^*$ such that for all $\mathbf{w}_0 \in U$, the sequence $F^k(\mathbf{w}_0)$ of iterates of $F$ converges to $\mathbf{w}^*$. Moreover, the rate of convergence is at least linear in $k$.
\end{theorem}
\begin{proof}[Proof of Theorem~\ref{thm:convergence}]
In order to apply Theorem~\ref{thm:ostrowski}, we define the map 
\begin{equation}
    F(\theta) \coloneqq \theta - \eta\left(I + \eta H_o(\theta)\right)^{-1} \xi(\theta),
\end{equation}
as the application of a single step of Algorithm~\ref{alg:pcgd}.
The fixed points of this map are exactly those points $\bar{\theta}$ where $\xi(\theta) = 0$.
According to Theorem~\ref{thm:ostrowski}, PCGD converges locally to such a point $\bar{\theta}$, if  
\begin{equation}
    \nabla F(\bar{\theta}) = I - \eta(I+\eta H_o(\bar{\theta})^{-1} H(\bar{\theta}) = I- \eta(I+\eta (A(\theta) + S_o(\theta)))^{-1} H(\bar{\theta}).
\end{equation}
has eigenvalues in the unit circle. Here, $A(\bar{\theta})$ is the antisymmetric part of the game Hessian $H(\bar{\theta})$ and $S_o(\bar{\theta})$ is the off-block-diagonal part of the symmetric part $S(\bar{\theta})$ of $H(\bar{\theta})$, so that $A(\bar{\theta})+S_o(\bar{\theta}) = H_o(\bar{\theta})$. 
Omitting from now on the argument $\bar{\theta}$, this holds if and only if the eigenvalues $\mu$ of $(I+\eta H_o)^{-1} H$ satisfy
\begin{equation*}
    |1-\eta \text{Re}(\mu) - i \eta \text{Im}(\mu)|^2 < 1
\end{equation*}
or equivalently,
\begin{equation*}
    \eta < 2\frac{\text{Re}(\mu)}{|\mu|^2}\,.
\end{equation*}
Assume $0 < \eta < 1/2\|S\|$ and let $\mu$ be any eigenvalue of $(I+\eta H_o)^{-1} H$ with normalized eigenvector $u$. By symmetry of $S$ and the antisymmetry of $A$, we can write $s = u^{*} S u$ and $ia = u^{*} A u$ with $a, s \in \mathbb{R}$. In particular, $H \succeq 0 \Leftrightarrow S \succeq 0$ implies $s \geq 0$.
Writing $S_d = S - S_o$ for the off-block-diagonal part of $S$, it holds moreover that
\[ -u^* S_d u \geq - \max_{\|v\| = 1} v^* S_d v \geq -\max_{\|v\| = 1} v^* S v = -\| S \| \, \]
and thus 
\[ u^* S_{o}u = u^* Su -u^* S_d u \geq -2\| S \| \,. \]
Here, we have exploited the general fact that the smallest eigenvalue of the block-diagonal part of a positive definite matrix is always at least as large as the smallest eigenvalue of the full matrix, as well as the assumption that $\bar{\theta}$ is a local Nash equilibrium and thus $S$ is positive definite.
The assumption that $\eta\|S\| < 1/4$ therefore implies $\delta \coloneqq \eta u^* S_o u > -1/2$. We obtain
\begin{align*}
    (I+\eta H_o)^{-1} H u &= \mu u \nonumber\\
    H u & = \mu(I+\eta (A + S_o)) u \nonumber\\
    s + ia & = \mu(1+ \delta + i \eta a ) \nonumber\\
    \mu &= \frac{s+ia}{1+ \delta + i \eta a} = \frac{(1 + \delta) s+\eta a^2}{(1 + \delta)^2+\eta^2 a^2} + i \frac{a(1 + \delta-s \eta)}{(1 + \delta)^2+\eta^2 a^2}.
\end{align*}
It follows that
\begin{equation*}
    2 \frac{\text{Re}(\mu)}{|\mu|^2} = 2 \frac{(1 + \delta)s+\eta a^2}{(1 + \delta)^2+\eta^2 a^2} \cdot \frac{(1 + \delta)^2+\eta^2 a^2}{s^2 + a^2} = 2\frac{(1 + \delta)s+\eta a^2}{s^2 + a^2}
\end{equation*}
and so
\begin{equation}\label{eq:eta_condition1}
    \eta < 2\frac{\text{Re}(\mu)}{|\mu|^2} \Leftrightarrow \eta(s^2 - a^2) < 2(1 + \delta) s \,.
\end{equation}
Note that we cannot have $s=a=0$ since $H$ is assumed to be invertible, so this always holds if $s=0$. Now assuming $s \neq 0$, the equation holds for any $\eta > 0$ if $s^2 - a^2 < 0$ since the LHS is negative and the RHS positive. Finally assume $s^2 - a^2 \geq 0$, and notice that $s \leq \|S\|$
implies 
\begin{equation*}
    \eta < \frac{1}{2\|S\|} \leq \frac{1}{2s} < \frac{1 + \delta}{s} < \frac{2(1 + \delta)}{s} = \frac{2(1 + \delta)s}{s^2} \leq \frac{2(1 + \delta)s}{s^2 -a^2}\,,
\end{equation*}
which is equivalent to \eqref{eq:eta_condition1}. We have shown that the eigenvalues lies in the unit circle for any $0 < \eta < \frac{1}{2||S||}$, and thus conclude the proof using Theorem~\ref{thm:ostrowski}.
\end{proof}

\section{Numerical experiment details}

\subsection{Snake Game Implementation}
% \begin{minipage}[t]{0.3\textwidth}
%   \centering\raisebox{\dimexpr \topskip-\height}{%
%   \includegraphics[width=\textwidth]{figures/snake_game.png}}
%   \captionof{figure}{A game of \emph{snake} between involving four agents.}
%   \label{fig:snake_game_appendix}
% \end{minipage}\hfill
% \begin{minipage}[t]{0.65\textwidth}
In the Snake game, four different snakes compete in a fixed 20x20 space to box each other out, to consume a number of randomly-placed fruits, and to avoid colliding with each other or the walls. Each snake is initialized with length 5, and can use their body to box out another snake. Snakes are initialized in either a clockwise or counter-clockwise spiral orientation, where each agent has a different initial orientation. We adapt the implementation of Snake by \citet{marlenv2021} by first rotating the observations such that the snake's current direction corresponds to the top of the multi-channel observation matrix. Similarly, we adjust the action space to instead be three actions (no change in direction or no-op, turn left, or turn right) relative to the snake's current orientation instead of the original five (no-op, up, down, left, right) global actions. 

At each step of the game, each snake observes the 20-by-20 space centered at its head and decides whether to continue moving forward, to turn left, or to turn right. The observation vector is structured as a 4x20x20 tensor: the first channel denotes any cells corresponding to the current snakes body as 1, zero elsewhere; the second channel denotes any nearby grid cell with an of opponent snake's body as 2, denotes any cells containing walls as 1, zero elsewhere. The third channel denotes any nearby fruits as 1, zero elsewhere, and the fourth channel denotes any cells containing nearby walls as 1, zero elsewhere. Agents are rewarded with $+1$ from consuming a fruit, $+20$ from capturing another snake, $-1$ for colliding with either another snake's body or a wall, and $+1$ for being the last snake alive. Once a snake collides with a wall or another snake, it is eliminated from the game and its body is removed from the playing field. We terminate the game as soon a single snake wins, since continuing to optimize would become a single agent optimization problem.

Each agent policy maps the observation vector of the game to a categorical distribution of 5 outputs using a network with two hidden layers (64-32-5). Agents sample actions from a softmax probability function over this categorical distribution. For all of SimGD, Extragradient, SGA and PCGD scenarios, we trained the agents for 15000 epochs. At each epoch, we collected a batch of 32 trajectories. 
%All parameters were the same between the SimGD training scenario and PCGD training scenario: 
We used a learning rate of 0.001 and GAE for advantage estimation with $\lambda=0.95$ and $\gamma=0.99$.

%\todo{TODO: Add snake game parameters for SGA and Extragradient}

\begin{figure}[ht]
    \begin{minipage}[t]{0.33\textwidth}
        \centering
        \includegraphics[width=1\linewidth]{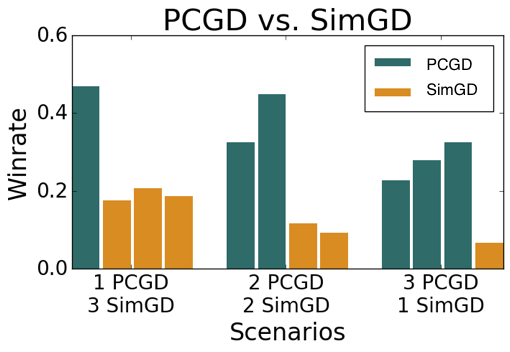}
    \end{minipage}\hfill
    \begin{minipage}[t]{0.33\textwidth}
        \centering
        \includegraphics[width=1\linewidth]{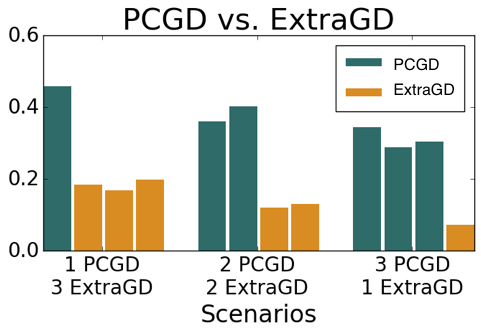}
    \end{minipage}\hfill
    \begin{minipage}[t]{0.33\textwidth}
        \centering
        \includegraphics[width=1\linewidth]{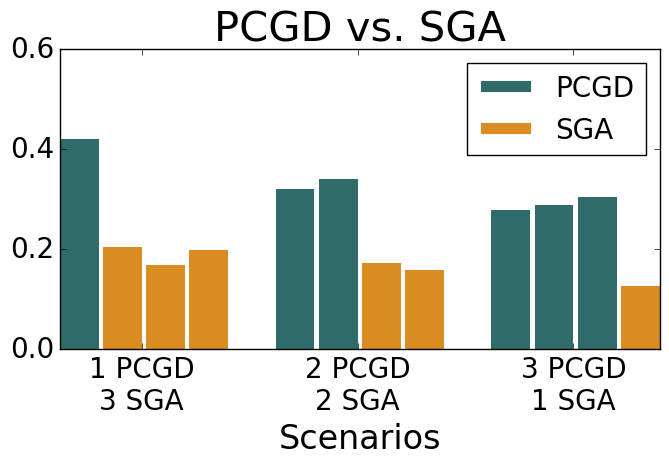}
    \end{minipage}\hfill
\caption{\textbf{Snake game, win rate with respect to most captures:} We plot the relative win rate of the different agents according to the alternative objective of catching or ``boxing out" the most snakes.}
\label{fig:snake_game_appendix}
\end{figure}

\begin{figure}[t]
\begin{center}
\includegraphics[width=0.4\linewidth]{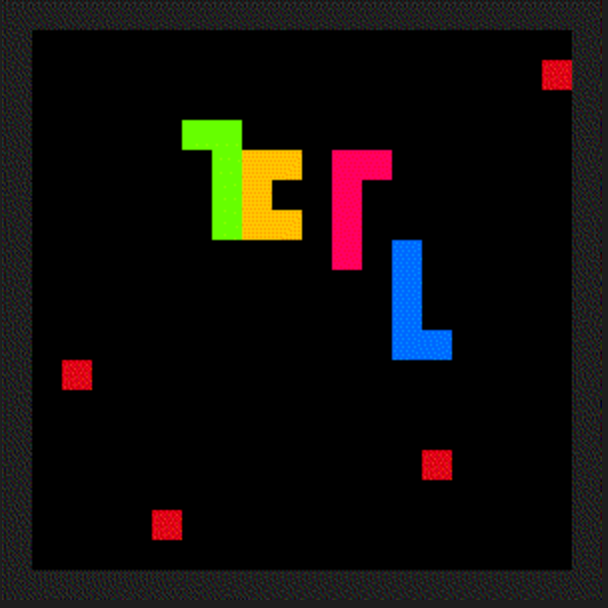}
\hfill
\includegraphics[width=0.4\linewidth]{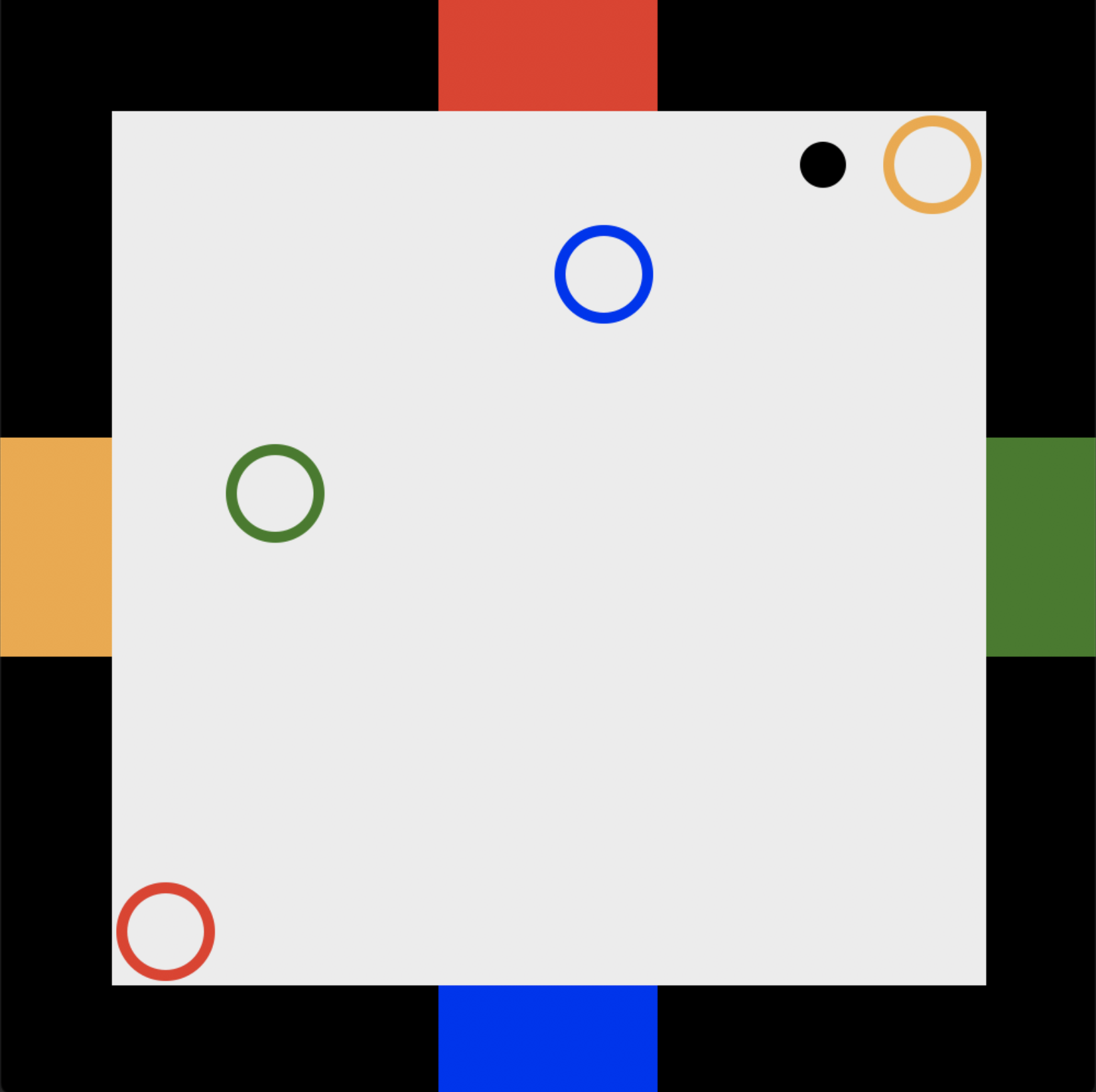}
\end{center}
\caption{A game of \emph{snake} between four agents (left) and four-player Markov Soccer (right). The goal of the snakes are to eat the fruit (in red) and force other snakes to run into them to ``capture'' them. In four-player Markov Soccer, each agent (colored circle) tries to obtain the ball (in black) and tries to place it in one of the other agents' goals.}
\end{figure}

\subsection{Markov Soccer Game Implementation}
% \begin{minipage}[t]{0.3\textwidth}
%   \centering\raisebox{\dimexpr \topskip-\height}{%
%   \includegraphics[width=\textwidth]{figures/markov_soccer.png}}
%   \captionof{figure}{Four-player Markov soccer game.}
%   \label{fig:soccer_game_appendix}
% \end{minipage}\hfill
% \begin{minipage}[t]{0.65\textwidth}
The setup of the four-player Markov soccer is shown in the left. It ist based on the two-player variant~\citep{conf/icml/Littman94} and is played between four players that are each randomly initialized in one of the 8x8 grid cells. The ball is also randomly initialized in one of the 8x8 grid cells. Each agent is assigned a goal and are supposed to pickup the ball (by moving into the ball's grid location or stealing it from another player) and place it in one of their three opponents goals. Agents are allowed to move \emph{left}, \emph{right}, \emph{up}, or \emph{down}, or to stand still. If one player holds the ball and another agent's move would put it into the position of the ball-holding agent, the second agent successfully steals the ball, and both agents keep their original positions. 
%\end{minipage}

%The setup of the four-player Markov soccer is shown in Figure \ref{fig:markov_soccer_game_example}. It extends from the two player variant to a four player and is played between four players, who are each randomly initialized in one of the 8x8 grid cells. The ball is also randomly initialized in one of the 8x8 grid cells. Each agent is assigned a goal and are supposed to pickup the ball (by moving into the ball's grid location or stealing it from another player) and place it in one of their three opponents goals. Agents are allowed to move \emph{left}, \emph{right}, \emph{up}, or \emph{down}, or to stand still. If one player holds the ball and another agent's move would put it into the position of the ball-holding agent, the second agent successfully steals the ball, and both agents keep their original positions. At each step of the game, agent actions are collected and processed in a random order; thus, it is possible that multiple steals can happen in a single round. The game ends if the player holding the ball moves into a goal, and the goal scoring player is rewarded with $+1$, the player who was scored on is penalized with $-1$, and all other players are penalized with $-0.25$.
At each step of the game, agent actions are collected and processed in a random order; thus, it is possible that multiple steals can happen in a single round. The game ends if the player holding the ball moves into a goal, and the goal scoring player is rewarded with $+1$, the player who was scored on is penalized with $-1$, and all other players are penalized with $-0.25$.

Given a clockwise order of players $O = [A, B, C, D]$, where $A$ is the location of the player with the top goal, $B$ with the rightmost goal, $C$ with the bottom goal, and $D$ with the left side goal, the local state vector of each player with respect to other players is defined as:
\begin{gather*}
    S_A = [x_{G_B}, y_{G_B}, x_{G_C}, y_{G_C}, x_{G_D}, y_{G_D}, x_{\text{ball}}, y_{\text{ball}}, x_{B}, y_{B}, x_{C}, y_{C}, x_{D}, y_{D}] \\
    S_B = [x_{G_C}, y_{G_C}, x_{G_D}, y_{G_D}, x_{G_A}, y_{G_A}, x_{\text{ball}}, y_{\text{ball}}, x_{C}, y_{C}, x_{D}, y_{D}, x_{A}, y_{A}] \\
    S_C = [x_{G_D}, y_{G_D}, x_{G_A}, y_{G_A}, x_{G_B}, y_{G_B}, x_{\text{ball}}, y_{\text{ball}}, x_{D}, y_{D}, x_{A}, y_{A}, x_{B}, y_{B}] \\
    S_D = [x_{G_A}, y_{G_A}, x_{G_B}, y_{G_B}, x_{G_C}, y_{G_C}, x_{\text{ball}}, y_{\text{ball}}, x_{A}, y_{A}, x_{B}, y_{B}, x_{C}, y_{C}]
\end{gather*}

where $x_P, y_P$ are defined as the relative position of the current player to some other item $P$ in horizontal and vertical offsets, and where $G_A, G_B, G_C, G_D$ are the goals of players $A, B, C$ and $D$. The observation state $O_P$ of each player $P$ is then as follows:
\begin{gather*}
    O_A = [S_A, S_B, S_C, S_D] \qquad\qquad
    O_B = [S_B, S_C, S_D, S_A] \\
    O_C = [S_C, S_D, S_A, S_B] \qquad\qquad
    O_D = [S_D, S_A, S_B, S_C] 
\end{gather*}

Each agent policy maps the observation vector of the game to a categorical distribution of 5 outputs using a network with two hidden layers, the first with 64 neurons, the second with 32. Agents sample actions from a softmax probability function over this categorical distribution. For all of SimGD, Extragradient, SGA and PCGD scenarios, we trained the agents for 20000 epochs. At each epoch, we collected a batch of 16 trajectories. 
%All parameters were the same between the SimGD training scenario and PCGD training scenario: 
We used a learning rate of 0.01 and GAE for advantage estimation with $\lambda=0.95$ and $\gamma=0.99$.

%\todo{TODO: Add soccer game parameters for SGA and Extragradient}

\begin{figure}[ht]
    \begin{minipage}[t]{0.33\textwidth}
        \centering
        \includegraphics[width=1\linewidth]{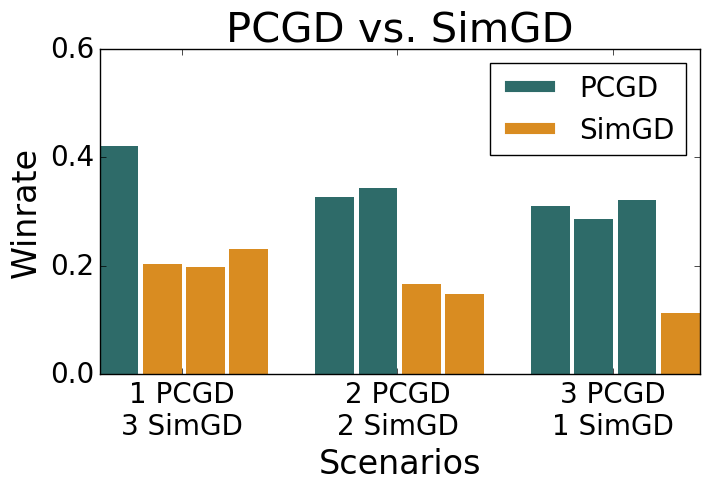}
    \end{minipage}\hfill
    \begin{minipage}[t]{0.33\textwidth}
        \centering
        \includegraphics[width=1\linewidth]{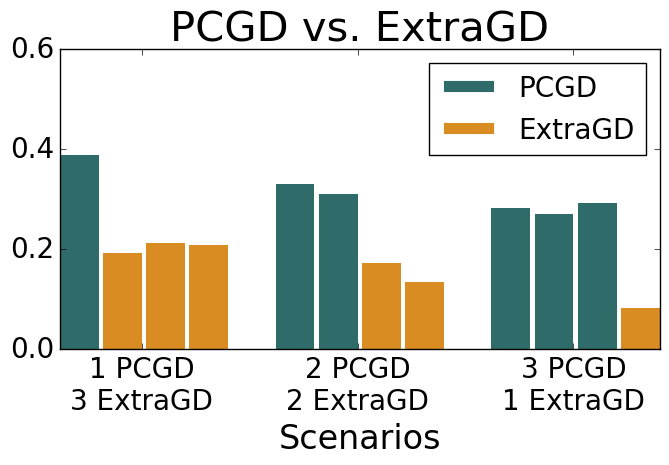}
    \end{minipage}\hfill
    \begin{minipage}[t]{0.33\textwidth}
        \centering
        \includegraphics[width=1\linewidth]{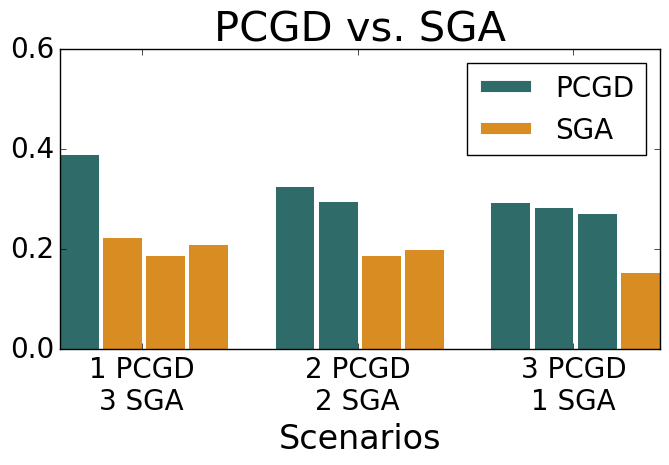}
    \end{minipage}\hfill
\caption{\textbf{Markov Soccer, win rate with respect to number of ``steals'':} We plot the relative win rate of the different agents according to the the number of times they steal the ball from another agent.}
\label{fig:markov_soccer_appendix}
\end{figure}

%\todo{TODO: Add additional graphs to show SGA and Extragradient are better than SimGD (to cover bases)}

\subsubsection{Electricity Market Game Implementation}
The implementation of electricity market game follows Fig.~\ref{fig:append_6bus}, where the system is a standard 6-bus system from~\citep{ruhi2017opportunities}. The demand at each bus is set to be $[150, 300, 280, 250, 200, 300]$. Each bus can have one or more generators assigned to it. In our experiments, we place one high marginal-cost generator (e.g., coal or gas generator) at each of the six buses, which has a fixed capacity bid of 1000 units and a constant marginal cost of 35 units. %For instance, as in Fig.~\ref{fig:append_6bus}, bus 2's generation cost is $35$ \$/MWh and generation output range is 0-1000 MW. 
This ensures that the market optimization problem is always feasible and makes the game environment more robust to randomly sampled bids learning agents might make during training. 
\begin{figure}[ht]
  \centering
  \includegraphics[width=0.6\textwidth]{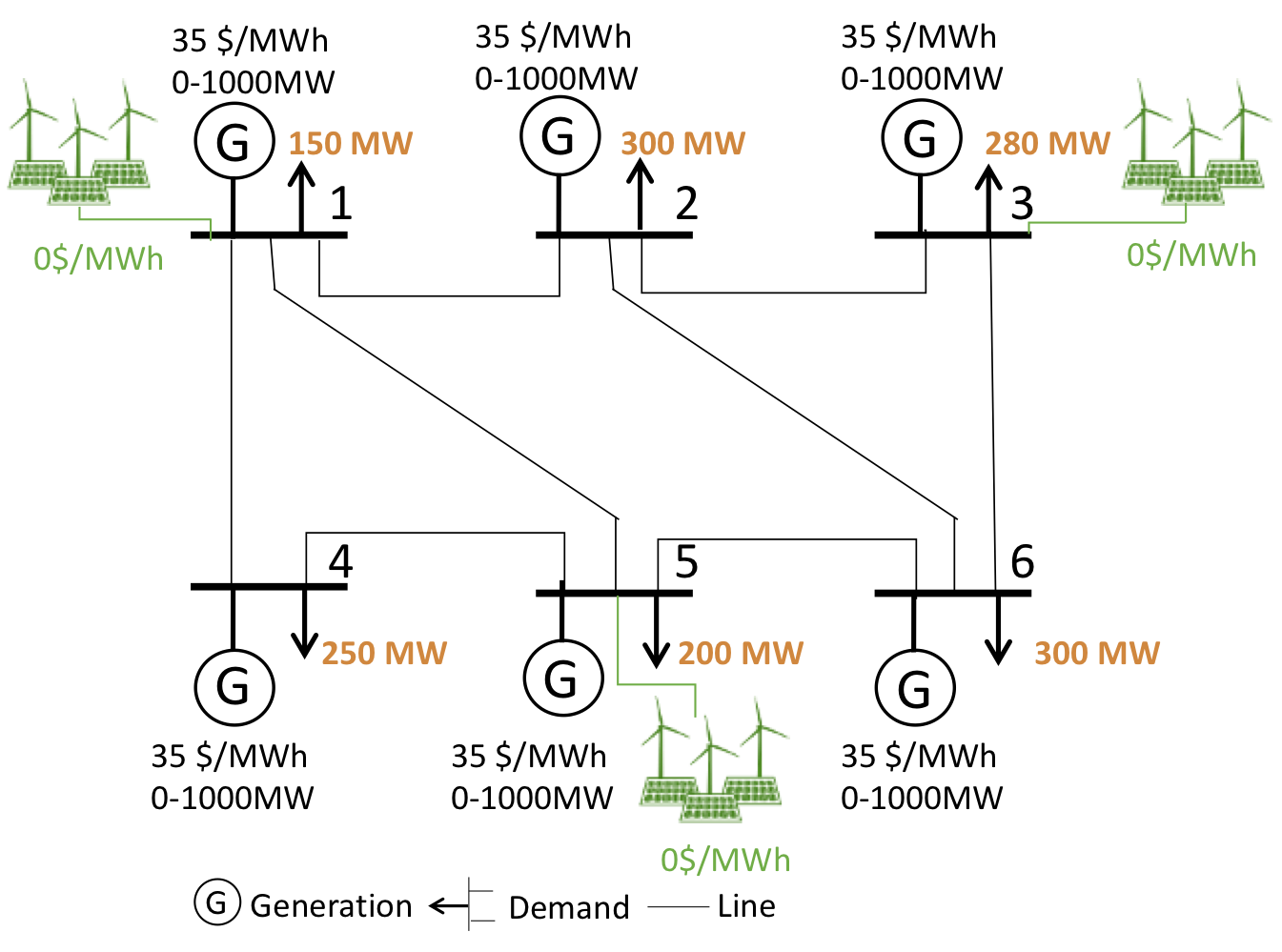}
  \caption{Illustration of the 6-bus market game.}
  \label{fig:append_6bus}
\end{figure}

We place the three learning generators at each of Buses 1, 3, and 5, and train the agents together under both PCGD, SimGD, ExtraGD and SGA. We further add a flag to indicate whether that bus is under load, affected by the price. Without loss of generality, this framework can simulate the effect of both elastic and inelastic electricity demand. The load flag is set based on the previously calculated electricity price (LMP) at that bus: if the price at that bus exceeds a certain threshold, the flag is set to 1; otherwise, it is set to 0. For inelastic demand single-stage repeated game, we set the LMP load thresholds at each bus as infinity. For elastic demand multi-stage game, we set the LMP load thresholds at certain positive values and can be different across buses. For our experiments, the LMP load thresholds for each bus are [25, 25, 25, 35, 30, 25]. Once the LMP at certain bus exceeds the LMP load threshold, the demand at that bus is reduced by half. 

Each trajectory of this electricity market game goes as follows: We initialize load states randomly (0 or 1) at each bus: standard demand or reduced demand status. For each step, each of the three learning agents observes the load state of the environment, a 6-vector containing the load state flags of Buses 1 through 6, and submits a bid for the maximum capacity $p_{i, t}^{\max}$ that they are willing to generate. This bid is capped from 0 to 10000 generation units. The market solver collects these bids and solves the market optimization problem in \eqref{eq:lmp} to determine: (i) the electricity price at each bus $\mathrm{LMP}_{k, t}$ and the (ii) quantity to be generated per generator $p_{i,t}^{*}$. From the results of the market optimization algorithm, we calculate profit for each learning agent and return it as immediate reward. Profit at each step is divided by a constant factor of 50 to stabilize learning in this game. At the end of each step, the electricity price at each bus is used to update system states (i.e., demand state) for the next game step. Finally, to create a finite horizon, at each step the game has a fixed probability $p$ of ending. For our experiments this probability is $p=0.2$, giving the game an expected length of 5 steps.

%Each trajectory of this electricity market game goes as follows: For each step, each of the three learning agents submits a bid for the maximum capacity they are willing to generate. This bid is capped from 0 to 10000 generation units. The market solver collects these bids and solves a minimization problem to determine: (1) the electricity price at each bus as specified in Eq. \ref{eq:min_cost} and the (2) quantity to be generated per generator. From the results of the market solver, we calculate profit for each learning agent and return it as immediate reward. Profit at each step is divided by a constant factor of 50 to stabilize learning in this game.  

Each agent's policy is parameterized as a Gaussian policy, where the mean is the output from 2 two-hidden layer neural network (128 neurons each layer) and the standard deviation is fixed at 25.
%a Gaussian distribution with fixed standard deviation of 25 using two hidden layers of 128 neurons each. 
Each player samples a bid for its maximum generation capacity from this Gaussian distribution. For all optimization methods, players were trained for 1000 episodes. In each epoch, we collected 64 trajectories per update step. We used a learning rate of 0.001 and GAE for advantage estimation with $\lambda=1$ and $\gamma=1$, as we optimize over non-discounted profit.

% \todo{Add parameters for baselines}

% \todo{Add graphs to show PCGD against other baselines, already uploaded, need to just show results}

% \todo{Add graphs comparing training steps vs time}.

\subsection{Computational Resources}
All experiments were implemented and trained on Amazon AWS P3 instances with an Nvidia V100 Tensor Core GPU.

\end{document}